%% file: Paper.tex
\documentclass{article}

\usepackage{arxiv}

\usepackage[utf8]{inputenc} 
\usepackage[T1]{fontenc}    
\usepackage{hyperref}       
\usepackage{url}            
\usepackage{booktabs}       
\usepackage{amsfonts}       
\usepackage{nicefrac}       
\usepackage{glossaries}
\usepackage{xspace}
\usepackage{amsthm}
\usepackage{bm}
\usepackage{tikz}
\usepackage{graphicx}
\usepackage{subcaption}
\usepackage{multirow}
\usepackage[inline]{enumitem}

\newcommand{\minfo}[1]{\mathbf{I}\left(#1\right)}
\newcommand{\eg}{e.g.\@\xspace}
\newcommand{\ie}{i.e.\@\xspace}
\newcommand{\sect}{Sect.\@\xspace}
\newcommand{\Prob}[1]{\ensuremath{P\left(#1\right)}}
\newcommand{\etal}{{\em et al}.\@\xspace}
\newcommand{\textbif}[1]{\textbf{\textit{#1}}}
\newcommand{\foottext}[1]{{\let\thefootnote\relax\footnote{#1}}}

\newtheorem{theorem}{Theorem}[section]
\newtheorem{lemma}[theorem]{Lemma}

\input{tikzbayes.code.tex}
\input{glossaries.tex}
\makeglossaries

\title{The Extended Dawid-Skene Model:\\{\Large Fusing Information from Multiple Data Schemas}}

\author{
  Michael P. J. Camilleri \\
  School of Informatics\\
  University of Edinburgh \\
   \And
 Christopher K. I. Williams \\
  School of Informatics\\
  University of Edinburgh\\
}

\begin{document}
\maketitle

\begin{abstract}
While label fusion from multiple noisy annotations is a well understood concept in data wrangling (tackled for example by the \acrfull{nam} model), we consider the extended problem of carrying out learning when the labels themselves are not consistently annotated with the same schema.
We show that even if annotators use disparate, albeit related, label-sets, we can still draw inferences for the underlying full label-set.
We propose the \acrfull{isac} to translate the fully-specified label-set to the one used by each annotator, enabling learning under such heterogeneous schemas, without the need to re-annotate the data.
We apply our method to a mouse behavioural dataset, achieving significant gains (compared with \acrshort{nam}) in out-of-sample log-likelihood (-3.40 to -2.39) and F1-score (0.785 to 0.864).
\end{abstract}

\keywords{Multi-schema learning \and Crowdsourcing \and Annotations \and Behavioural characterisation \and Probabilistic modelling \and Data wrangling} 

\foottext{This article has been published in P.\ Cellier and K.\ Driessens (Eds.): ECML PKDD 2019 Workshops, CCIS 1167, pp. 121 -- 136, 2020. The final authenticated publication is available online at \url{https://doi.org/10.1007/978-3-030-43823-4_11}}

\section{Introduction}
\label{S_INTRODUCTION}

Machine learning is based on learning from examples \cite{MISC_023}.
This often requires humans annotation, \eg class labels for image classes in ImageNet \cite{MISC_015}.
However, human labelling is error prone and consequently, methods such as the \gls{nam} model \cite{AMD_012} have been developed to estimate individual error rates and draw inferences on the true label from multiple annotators, see \eg \cite{AMD_019, AMD_017}.

In this paper, we are interested in the extended problem of carrying out such learning when the annotations have been carried out under different schemas, and in so doing, help to automate the data wrangling and cleaning portion of data science.
Given a `complete' set of possible labels, we consider the scenario where the annotations for different samples are performed using different subsets (schemas) of this `complete' label-set.
A schema can be obtained, for example, by aggregating labels together to produce fewer, coarser labels, or by singling out one label to annotate and lumping all the others together (\ie `One-vs-Rest').
This is a common data wrangling problem in scientific analysis where the actual nature of the research question is being formulated: for example, in labelling animal behaviour, scientists may realise half-way through data collection that a certain activity is rich enough that it warrants splitting into multiple labels.
Alternatively, due to the expertise of certain annotators, they may be directed to focus on specific subsets of activity, and clumping all others.

The challenge we address here is how to draw inferences about the underlying complete label-set, despite being provided with annotations which make use of different labelling schemas.
Normally, this would not be possible without re-annotating the entire data-set (which is often expensive) or simply discarding older data (which is wasteful in limited data scenarios).
Our contribution is to show that with the appropriate formulation, learning from all the data can indeed be achieved by adding an \gls{isac} which allows us to translate the full label-set to the one used by a given annotator.
Furthermore, we demonstrate the applicability and effectiveness of our method for behavioural annotation, using both simulated and actual data.

The rest of our paper is structured as follows. In \sect \ref{S_PROBLEM_DEFINITION} we define the data wrangling problem we tackle and propose our solution, and then compare our approach to related work (\sect \ref{S_RELATED_WORK}). Subsequently we describe a concrete problem which motivated our model in \sect \ref{S_DATA_DESCRIPTION}, and in \sect \ref{S_EXPERIMENTS} report experimental results under various scenarios. We conclude with a discussion of the merits of the model and proposed future extensions.

\section{Problem and Model Definition}
\label{S_PROBLEM_DEFINITION}

We start by defining a `complete' set of labels $L = \lbrace1, ..., |L|\rbrace$ encompassing all possible classes/feature values, which we will refer to as the `full label-set'.
However, we consider the case where the observations are drawn from a reduced sub-set of $L$.
That is, given $|S|$ different label-sets/schemas, denoted $L_s$ for $s \in \lbrace 1, ..., |S|\rbrace$, different samples are labelled according to different schemas.
Each $L_s$ may contain labels from $L$ and/or groupings thereof, as illustrated in Fig. \ref{FIG_SCHEMA_EXAMPLE}.

\begin{figure}
	\centering
	\includegraphics[width=0.8\textwidth]{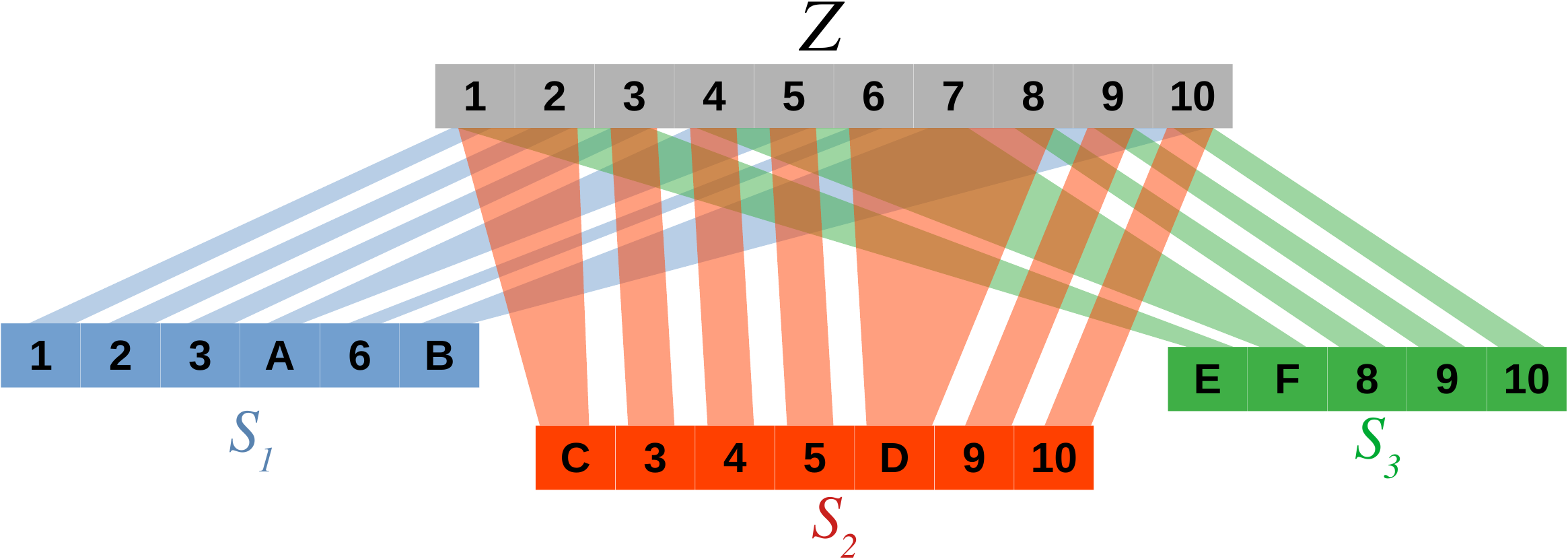}
  	\caption{An example scenario with ten labels, and three schemas (colour-coded), showing how super-labels (enumerated \textbf{A} through \textbf{F}) are constructed from the full label-set. Note that while the above super-labels encompass contiguous labels, this is only for clarity and need not be the case.}
  	\label{FIG_SCHEMA_EXAMPLE}
\end{figure}

To motivate our problem consider the task of documenting the behaviour of an individual according to a discrete set of labels.
A number of annotators are tasked to do this, but their annotations are not restricted to a single schema.
Our aim is to collate their labels so that we get a posterior belief over the true behavioural state, and to do so while constructing a global rather than a single model per-schema, thus sharing statistical strength across the entire data-set.
In what follows, we first describe the \gls{nam} model, which can be used to solve the problem under the constraint of a single schema, and then show how using \gls{isac} we can achieve an Extended \gls{nam} model for dealing with multiple schemas.

\subsection{Model Definition}

The standard \gls{nam} model appears in Fig. \ref{FIG_DS}.
The categorical variable $Z$ represents the true behaviour of the individual, and is parametrised by the prior $\pi$ over the full label-set (indexed by $z$).
$U_k$ is the \emph{observed} annotation provided by each annotator $k$, and models the observer error-rate through a \gls{cpt} ($|U|=|Z| = |L|$):
\begin{equation}
\psi_{k,u,z} \equiv \Prob{U_{k} = u | Z = z} , \label{EQ_PSI_DEFINITION}
\end{equation}
where the subscripts indicate indexing in the respective dimension.

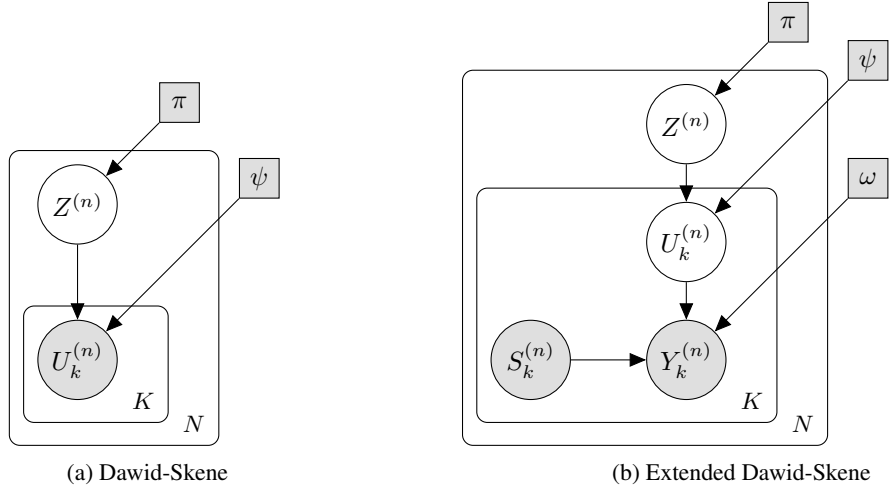
\begin{figure}
	\centering
	\begin{subfigure}[b]{0.4\textwidth}
		\centering
		\input{model_ds}
  		\caption{\acrlong{nam}}
  		\label{FIG_DS}
  	\end{subfigure}
  	\hspace{2em}
  	\begin{subfigure}[b]{0.45\textwidth}
		\input{model_isac}
  	\caption{Extended \acrlong{nam}}
  	\label{FIG_ISAC}
  	\end{subfigure}
  	\caption{Multi-Annotator Label fusion with the (a) \acrshort{nam} and (b) Extended \acrshort{nam} (using \acrshort{isac} adapter) models.}
  	\label{FIG_ISAC_VS_DS}
\end{figure}

In our setup (Fig. \ref{FIG_ISAC}), $U_k$ is `corrupted' by the schema: \ie we only observe $Y_k$ whose domain is conditioned by the schema $S_k$.
$Y_k$ is another discrete variable representing the annotator's assigned label contingent on which schema $S^{(n)}_k$ is currently active: \ie $|Y^{(n)}_k| = |L_{S^{(n)}_k}|$.
We assume knowledge of $S^{(n)}_k$, a valid assumption in our application domain.
The mapping from $U$ to $Y$ (conditioned on $S$) is modelled by the \acrfull{isac} \gls{cpt} $\omega$:
\begin{equation}
\omega_{y,u,s} = \Prob{Y = y | U = u, S = s} . \label{EQ_OMEGA_DEFINITION}
\end{equation}
For our purposes, $\omega$ is \emph{fixed} and \emph{deterministic}: \ie all entries are either 1 or 0, and encode expert knowledge about which labels in $L$ map to the same $L_s$.
This gives a very intuitive way to construct the mapping, as governed by:
\begin{equation}
\omega_{y,u,s} = \begin{cases}
 1 & \text{if $u$ is one of the states captured by $y$ under $s$ ,} \\
 0 & \text{otherwise}.
\end{cases}
\end{equation}
If we assume a one-hot-encoding of the variables (such that a particular manifestation is indicated by indexing the respective variable dimension), we can represent our model by the following joint distribution:
\begin{align}
& P(Z,U,Y|S;\Theta, \omega) = \prod_{n=1}^N\prod_{z=1}^{|L|}\left(\pi_z\prod_{k=1}^K\prod_{u=1}^{|L|}
\left\lbrace \psi_{k,u,z} M_{\omega}^{(n)}\left(k,u\right) \right\rbrace^{U^{(n)}_{k,u}}\right)^{Z^{(n)}_z} , \label{EQ_ISAC_LIKELIHOOD}
\end{align}
where $\Theta = \lbrace\pi,\psi\rbrace$, and we have defined the \gls{isac} message:
\begin{align}
M_{\omega}^{(n)}\left(k,u\right) \equiv \prod_{s=1}^{|S|} \left[ \prod_{y=1}^{|Y^{(n)}|} \omega_{y,u,s}^{Y_{k,y}^{(n)}} \right]^{S_{k,s}^{(n)}} . \label{EQ_ISAC_MSG_OMEGA}
\end{align}
Despite the dependence of $Y_k^{(n)}$ on $S^{(n)}_k$, we can standardise the annotator labels using a super-space $Y$ which encapsulates all the labels in the full label-set as well as any valid groupings thereof, as indicated by Lemma \ref{LEM_SUPER_SPACE} (see Appendix).

The proposed architecture allows the parameters $\Theta$ to model the data-generating process, while the inter-schema differences are captured by the emission probabilities $\omega$.
In doing so, we incorporate knowledge about the schema mapping, specifically as to which labels will map to which super-labels without effecting estimation of reliability metrics.
It is important to note that $\omega$ is annotator-independent, which reduces the model dimensionality and forces all inter-annotator variability to be incorporated in $\psi$.
Due to the \gls{isac} adapter, the Extended \gls{nam} model is able to infer more accurate statistics about the distribution of the full label-set, even in cases where the signal is very sparse (such as one vs rest schemas, see \ref{SS_PARAM_ESTIM} below).
Despite being deterministic, $\omega$ does not preclude multiple latent states mapping to any super-label, and hence, the model is rich enough to capture the inherent uncertainty over the latent state.

\subsection{Training the Model}
Training the model involves learning the parameters $\pi$ and $\psi$ ($\omega$ is fixed).
We add a \emph{Log Prior} to the log of the joint likelihood (Eq. \ref{EQ_ISAC_LIKELIHOOD}), and compute \gls{map} rather than \gls{mle} estimates for $\pi$ and $\psi$, thus reducing the risk of overfitting. 
We use conjugate Dirichlet priors:
\begin{align*}
Dir(\pi|\mathbf{\alpha}^\pi) \equiv & \frac{1}{\bm{\beta}(\alpha^\pi)}\prod_{z=1}^{|Z|}\pi_z^{\alpha^\pi_z-1} , \qquad\qquad Dir(\psi_{k,z}|\mathbf{\alpha}^\psi_{k,z}) \equiv \frac{1}{\bm{\beta}(\alpha^\psi_{k,z})} \prod_{u=1}^{|U|}\psi_{k,u,z}^{\alpha^\psi_{k,u,z}-1} .
\end{align*}

We derive an \gls{em} algorithm \cite{AMD_015} to infer the parameters.
During the E-step, we need to compute two expectations:
\begin{equation}
\gamma^{(n)}_z \equiv \left\langle Z^{(n)}_z \right\rangle = \frac{\pi_z M_{\psi}^{(n)}(z) } {\sum\limits_{z'=1}^{|Z|} \pi_{z'} M_{\psi}^{(n)}(z') } ,
\end{equation}
and
\begin{equation}
\rho^{(n)}_{k,u,z} \equiv \left\langle Z^{(n)}_z U^{(n)}_{k,u} \right\rangle = \frac{\pi_z M_{\omega}^{*(n)}(k,u,z) M_{\psi}^{*(n)}(k, z)  } {\sum\limits_{z'=1}^{|Z|}\left( \pi_{z'} M_{\psi}^{(n)}(z') \right) } ,
\end{equation}
where $M_{\omega}^{(n)}$ is as defined before (Eq. \ref{EQ_ISAC_MSG_OMEGA}), with messages:
\begin{align}
M_{\omega}^{*(n)}(k,u,z) & \equiv \ \psi_{k,u,z} M_\omega^{(n)}(k,u) \label{EQ_ISAC_MSG_XI_START} ,\\
M_{\psi}^{(n)}(z) \quad & \equiv \ \prod_{k'=1}^{K}\sum_{u'=1}^{|U|}M_{\omega}^{*(n)}(k',u',z) \label{EQ_ISAC_MSG_PSI} , \\
M_{\psi}^{*(n)}(k, z) & \equiv \frac{M_{\psi}^{(n)}(z)}{\sum_{u=1}^{|U|}M_{\omega}^{*(n)}(k,u,z)} \label{EQ_ISAC_MSG_PSI_STAR} .
\end{align}
The M-Step involves maximising the expected complete data log-likelihood with respect to each of the unknown parameters $\pi$ and $\psi$:
\begin{equation}
\hat\pi_z = \frac{\sum_{n=1}^N\gamma^{(n)}_{z} + \alpha^\pi_z - 1} {N + \sum_{z'=1}^{|Z|}\alpha^\pi_{z'} - |Z|} , \label{EQ_ISAC_PI_MAP}
\end{equation}
and
\begin{equation}
\hat{\psi}_{k,u,z} = \frac{  \sum_{n=1}^N\rho_{k,u,z}^{(n)} + \alpha^\psi_{k,u,z} - 1  }{  \sum_{n=1}^N\sum_{u'=1}^{|U|}\rho_{k,u',z}^{(n)} + \sum_{u'=1}^{|U|}\alpha^\psi_{k,u',z} - |U|  } . \label{EQ_ISAC_PSI_MAP}
\end{equation}
The full derivations are available in \cite{MISC_009}.

As regards computational complexity, we note that the \gls{isac} adapter acts as a message function in graphical modelling terms, and given that $\omega$ is fixed and both $Y$ and $S$ are observed, $M_\omega$ can be computed once and used throughout the optimisation: moreover, being a merely indexing operation, it is linear in the number of samples.
As regards estimation of the other parameters, each \gls{em} step scales linearly in the number of samples and annotators, and quadratically in the size of the super-schema.

\section{Related Work}
\label{S_RELATED_WORK}

Our approach towards automating information fusion is concerned with probabilistic inference from incompletely specified data.
In this respect, \gls{isac} is related to the general \gls{tl} field, specifically in learning across feature-spaces.
Our solution is novel in that it is applied to an `unsupervised' learning scenario, and rather than focusing on learning the mapping between feature spaces -- refer to \cite{TML_013} for a review -- we take the problem one step further through our interpretation of the different schemas (using domain knowledge about the specific problem), which allows us to collate information across label-spaces in an efficient manner.

Another perspective comes from \gls{mtl} \cite{TML_012}.
To relate to this literature we can view each label schema as a ``task'', but the analogy is not perfect.
In multi-task learning one aims to improve the learning of a model for each task by using knowledge contained in all or some of the tasks, while in our case, we typically consider a single task but the feature-space is only partially observed (by way of the schema). 
We do share a similar goal of sharing statistical strength across schemas (rather than across tasks): by using \gls{isac} we seek to \emph{fuse} the information from all annotators (who may be using different schemas) in order to draw inferences for the `complete' label-set (rather than one `task' at a time), and hence is a step beyond the standard \gls{mtl} setting.

Our schema mapping can be viewed as ``data coarsening'' as discussed in \cite{AMD_023}: however, our problem setup is different and applied to categorical rather than continuous data.
Cout \etal \cite{AMD_029} have addressed a similar problem, using a discriminative rather than generative method, but only applied for supervised learning.

One may be inclined to cast our problem into the hierarchical classification framework \cite{MISC_010}, particularly \gls{hmlc} \cite{MISC_025} due to the apparent `multi-label' aspect of the mapping together with the two-level nature.
While hierarchical classification seeks to structure the space of labels hierarchically according to a fixed taxonomy, we stress that contrary to multi-label classification, in our setting, there is a single valid label, but there is uncertainty on which one it is (due to the coarse labelling imposed by the schema).
Moreover, while we seek mandatory leaf-node predictions \cite{MISC_010}, we do not require specification of the full label hierarchy for each sample which to our knowledge has not been tackled before.
Finally, our model focuses on handling multiple annotators and their uncertainty.

\section{Description of the Data}
\label{S_DATA_DESCRIPTION}
We tested our Extended \gls{nam} on a social behaviour-phenotyping data-set for caged mice, obtained from \gls{mrc} as documented in \cite{CBD_026}.
This consists of 27.5 hours of annotated behaviour of various mice kept in cages of three.
Each 30-minute segment was individually labelled by two or three annotators from a pool of eleven individuals, with responsibility for annotating changing between segments.

Labelling involves specifying the start/end-times of exhibited behaviours, which periods are then aligned to 1-second boundaries.
The data contains periods for which no label is given (`unlabelled samples'): this is because the annotators were explicitly instructed not to annotate observations if they cannot be coerced to any of the available behaviours or if they were unsure about it.
The annotations follow one of four schemas (denoted \textbf{I}, \textbf{II}, \textbf{III} and \textbf{IV}), containing the labels shown in Fig.\ \ref{FIG_ANNOTATION_SCHEMAS}.
The schemas are consistent within a segment (\ie all annotators use the same schema) but change between segments.
The goal of the study is to infer the `true' latent behaviour of the mice given the observations, which can then be used for example in the analyse of phenotype differences between strains (although in this paper, all wild-type strains were used).
In this scenario, the need for a holistic model is even more significant since some labels are missing entirely from some schemas, and hence a model trained solely on data from a particular schema would miss potentially significant behaviour.

\begin{figure}[!ht]
\centering
\includegraphics[width=\columnwidth]{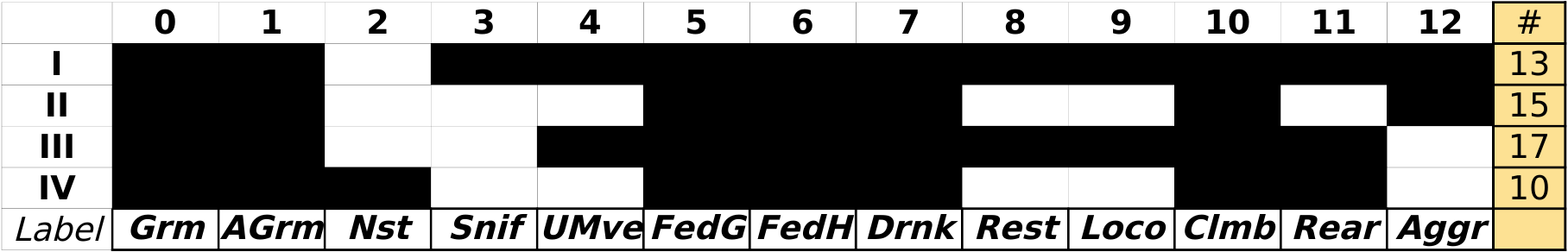}
\caption{The Behavioural Annotation Schemas used in this project, with \emph{black} cells indicating which labels (numerical representation, top row) are present in which schema (first column). The last row, marked \emph{(Label)} is our short-hand notation for referring to the labels while the last column indicates the number of segments in our data-set corresponding to each schema.}
\label{FIG_ANNOTATION_SCHEMAS}
\end{figure}

Since the schemas used did not have an explicit label to indicate a behaviour not in the label-set (\gls{nis}), we had to infer this from the unlabelled data.
We distinguish between two cases of such samples:
\begin{enumerate}
\item Informative Unlabelling, which arises from the observed behaviour not being in the schema (translating to \gls{nis}), \textsl{and}
\item Missing Data, \ie where the annotator was unsure about how to label a behaviour.
\end{enumerate}
We assign \gls{nis} only to those time-points where all responsible annotators do not give a label, treating all other unlabelled samples as \gls{mar} \cite{MISC_011}.
This is based on the assumption that in our laboratory setting, the annotators are adequately trained, and hence, the probability of all responsible annotators not providing a label is close to insignificant.
This was indeed verified by recording the fraction of time-points with no labelling by schemas, and saw that this is correlated with the schema (dropping to virtually 0 for Schemas I and III).

\section{Experimental Analysis}
\label{S_EXPERIMENTS}

We now document the empirical results which serve to illustrate the validity of the \gls{isac} method. Specifically we seek to answer two questions:
\begin{enumerate*}[label=(\alph*)]
\item can such a model be learnt under the condition of disparate schemas, \textsl{and, if so,}
\item is there merit to using \gls{isac} over just discarding incompatibly annotated data?
\end{enumerate*}
To this end, we report two main experiments.
Lacking any ground truth in the real data, we first evaluate the observed-data log-likelihood under both our Extended \gls{nam} and individual \gls{nam} models trained on each schema in \sect \ref{SS_EVAL_REAL}.
Next we analyse the ability of our model to learn the true data-generating process by evaluating the predictive performance on synthetic data for which ground-truth is available (\sect \ref{SS_LATENT_INFER}).
We also provide results on parameter recovery (\sect \ref{SS_PARAM_ESTIM}) as well as an information-theoretic analysis of the schema adapter (\sect \ref{SS_MUTUAL_INFO}).
In all our tests (except for the parameter recover), we train on a portion of the data and report measures on `unseen' data using cross-validation.

All our experiments were carried out on a desktop running Ubuntu Linux (18.04), with an Intel Xeon E3-1245 processor (4-cores) running at 3.5GHz, and 32Gb of memory.
The longest experiment (latent-state synthetic inference for the full data size repeated 20 times, each with 10-fold cross-validation) took about 6 hours. The code to replicate all results is available at \url{https://github.com/michael-camilleri/ISAR-Inter_Schema_AdapteR}

\subsection{Experimental Setup}
\label{SS_EXP_SETUP}

We explored training the models from multiple random restarts.
However, extensive testing indicated that starting from a diagonally-biased emission matrix ($\psi$) provided consistently better validation-set likelihoods: paired t-test with 164 \acrshort{dof} yielded a t-statistic of 2.90 (p=0.004) when compared to the best of 30-random restarts.
We hence initialised $\psi$ as a strongly diagonal matrix by adding a uniform matrix of 0.01 entries to the Identity matrix, and then normalising across $u$ to produce valid probabilities.
This encodes our belief that most annotators are consistent in their labelling (\ie most of the probability mass is on the diagonal).
It also provides the added benefit that the latent-states are `naturally' identifiable, avoiding the `label-switching' issue \cite{AMD_021} in the latent space since it biases the search in the vicinity of the identity permutation.
The prior $\pi$ was initialised to the uniform distribution (\ie all states equally likely).
In all cases, we used symmetric Dirichlet priors ($\alpha=2$) on the parameters $\pi$ and $\psi$.

\subsection{Likelihood-based Evaluation on Real Data}
\label{SS_EVAL_REAL}

We first evaluate our architecture on the motivating task of inferring latent mouse-behaviour from noisy annotations, and compare it to the \gls{nam} baseline trained independently per schema.
In the absence of ground-truths, we evaluate both models based on the evidence log-likelihood on out-of-sample data, using cross-validation with 11-folds\footnote[1]{The reason for this count is due to the natural groupings of segments in the available data.} (training on ten and evaluating metrics on the remaining one).
When training the \gls{nam} model, the individual schemas were augmented with the \gls{nis} label, to provide an equivalent observed sample-space, and allow for a like-with-like likelihood comparison between the two models on the same label-space.
The folds were engineered to be as uniform in size as possible while separating different mice in different folds to achieve more generalisable performance measures.
We:
\begin{enumerate*}[label=(\alph*)]
\item use fixed-folds, to provide a fair comparison between models, \textsl{and}
\item we report and compare measures on a per-segment basis, since the \gls{nam} architecture can only be trained on a single-schema at a time.
\end{enumerate*}

Table \ref{TAB_VAL_NLL} reports schema-averages (across segments) for the per-sample evidence log-likelihood (where the log-likelihood is divided by the number of samples), and the global average (computed across all segments).
All likelihoods are higher (better) in the \gls{isac} case, indicating the ability of the model to share statistical strength across the schemas, including learning about annotators which would otherwise not be observed in some schemas.
Specifically, a paired t-test with 54 \acrshort{dof} (55 segments) indicated a significant increase in validation-set log-likelihood for the \gls{isac} model as compared to the \gls{nam} model: the result yielded a t-statistic of 5.78 (p=$3.89\times10^{-7}$).

\begin{table}[!ht]
\centering
\caption{Validation Average Log-Likelihoods (higher is better)}
\begingroup
\setlength{\tabcolsep}{5pt}
\begin{tabular}{@{}lcccccc@{}} \toprule[1.5pt]
			& \multicolumn{4}{c}{\textbf{Schema}}			&  					&					\\ \cmidrule(r){2-5}
			& I			& II		& III		& IV		& \textbif{Mean}	&	\textbif{Std}	\\ \midrule
\gls{nam}	& -3.27		& -3.23		& -2.65		& -5.07		& -3.40				& 1.20				\\
\textbf{\gls{isac}}	& \textbf{-2.68}  	& \textbf{-2.29}  	& \textbf{-2.48}		& \textbf{-2.01}		& \textbf{-2.39}				& \textbf{0.82}				\\ \bottomrule[1.5pt]
\end{tabular}
\endgroup
\label{TAB_VAL_NLL}
\end{table}

\subsection{Latent State Inference in Synthetic Data}
\label{SS_LATENT_INFER}

While the real data lacks ground-truth with regards to the true mouse behaviour, we can simulate data using the parameters learnt above (to be as realistic as possible) and evaluate the \gls{map} `predictive' performance on it.
Note that in this scenario, we cannot compare \gls{isac} to the \gls{nam} model trained individually per-schema, since in every schema, \gls{nam} does not have knowledge of the entire label-set.
In effect, the \gls{nam} model cannot be used in such a scenario to give true predictions.
A naive alternative is to clump together all the samples as if they come from the same schema, and treat \gls{nis} as missing data.
This is based on the clearly incorrect assumption that the missing data is \gls{mar} and can thus be ignored, which will in general lead to inferior results.
It does however provide a baseline comparison.

In order to test the merits of \gls{isac} under a number of statistical conditions, we performed a study in which we simulated different data generation conditions.
The full details of the experimental procedure as well as the results are given in Appendix \ref{APP_ABLATION_STUDY}: below we report the case on the statistics which most closely match our dataset.
We ran the experiment 20 times, with 10-fold cross-validation in each run and report the macro-averaged F1-score (computing F1-score for each class independently and then averaging \cite[pp.~ 185]{MISC_013}) and predictive log-likelihood (log-likelihood assigned to the true label) in Table \ref{TAB_ACCURACY_BELIEF}.
We prefer the macro-averaged F1 score over accuracy, as we have high class imbalance, but care about each label equally. Note how in both metrics \gls{isac} shows consistently and (statistically) significantly better performance.

\begin{table}[!ht]
\centering
\caption{Macro F1 and predictive log-likelihood for the \gls{isac} and \gls{nam} models applied to synthetic data.}
\begingroup
\setlength{\tabcolsep}{5pt}
\begin{tabular}{@{}lcccccc@{}} \toprule[1.5pt]
 			& \multicolumn{3}{c}{Log-Likelihood}						    			& \multicolumn{3}{c}{Macro F1}	                                  \\\cmidrule(l){2-4}\cmidrule(l){5-7}
 			& \textbf{Mean}	    & \textbf{Std}	& \textbf{\textit{p}-value}				& \textbf{Mean}	    & \textbf{Std}	& \textbf{\textit{p}-value}	\\\midrule
\gls{nam} 	& -0.61			    & 0.09			& \multirow{2}{*}{6.1$\times10^{-16}$}	& 0.785			    & 0.02			& \multirow{2}{*}{3.2$\times10^{-26}$}	\\
\gls{isac} 	& \textbf{-0.35}	& \textbf{0.04} &  										& \textbf{0.864}	& 0.02	        & \\\bottomrule[1.5pt]
\end{tabular}
\endgroup
\label{TAB_ACCURACY_BELIEF}
\end{table}

\subsection{Parameter Recovery from Synthetic Data}
\label{SS_PARAM_ESTIM}

Another indicator of performance is the ability of our architecture to learn the `true' parameters which generate the data. We again generated synthetic data from known fixed values for $\Theta=\lbrace\pi, \psi\rbrace$ (obtained from the parameters trained on the real-data), and trained our model on it. While space precludes us from a full treatment of these results here, we observed convergence towards the same $\pi$ identified by using the full schema (up to 2.3\% error) even in extreme one-vs-rest schemas where the annotator only provides the presence/absence of a single label: $\psi$ was estimated to within 11.3\% of the true values under the \gls{mrc} schemas. More details are provided in Appendix \ref{APP_PARAM_RECOVERY}.

\subsection{Analysis of \acrlong{mi}}
\label{SS_MUTUAL_INFO}

We sought to explain the relative performance of the \gls{isac} architecture in terms of the \gls{mi} $\minfo{Z;Y}$ between the latent state $Z$ and sets of observations $Y$ from different schemas on the model fitted from the real data.
When using more than one schema we can also compute the Redundancy $R(Z;Y)$ \cite{MISC_012} between $Z$ and $Y$, where:
\begin{equation}
R(Z;Y) \equiv \sum_{s=1}^{|S|}\minfo{Z;Y_s} - \minfo{Z;Y_1,...,Y_{|S|}}
\end{equation}

\begin{table}
\centering
\caption{\acrlong{mi} $\mathit{\mathbf{I}}$ and Redundancy $R$ between the observations and the latent behaviour, under the effect of the different schemas. The statistics are reported across annotators.}
\begingroup
\setlength{\tabcolsep}{5pt}
\input{table_mutualinfo}
\endgroup
\label{TAB_MUTUAL_INFO}
\end{table}

The resulting measures are shown in Table \ref{TAB_MUTUAL_INFO}.
We consider first the \gls{mi} for individual schemas in Table \ref{TAB_MUTUAL_INFO} (top). Note how schema \textbf{I} yields the same \gls{mi} as if we had access to the full label-set: this is because in \textbf{I} there is only one missing label, and hence the model correctly identifies \gls{nis} with that label. Looking at the individual schemas, we see that those with a smaller number of labels coded as \gls{nis} have a higher mutual information.

We next consider combinations of the schemas, a subset of which appear in Table \ref{TAB_MUTUAL_INFO}.
That is we potentially have observations from up to four schemas for the same underlying latent state.
The table shows that (as expected) increasing the number of schemas yields higher mutual information, up to the maximum from using all four schemas.
We can also measure the redundancy of the different schemas.
This shows that the schemas are redundant (rather than synergistic), which makes sense given the way the model is constructed.

\section{Conclusions}
\label{S_DISCUSSION}

In this paper we have presented a novel and effective solution to inferring latent variables from observations across different but related label-spaces (schemas).
We developed an inter-schema adapter (\gls{isac}), that allows us to build a holistic model and share statistical strength across disparately-labelled portions of the data-set.
We validated our model under both simulated and real-world conditions, for a behaviour annotation task.
The \gls{isac} model improved on the baseline \gls{nam} in terms of evidence log-likelihood with an increase from -3.40 to -2.39.
In simulated data, \gls{isac} achieved a 10\% increase in macro F1-score.

While above we assume that the samples are \gls{iid}, we can easily extend the unsupervised model to the temporal modelling domain: indeed, we investigated such an extension in \cite{MISC_009}.
We have constructed the schema adapter from knowledge of the schemas and how labels are mapped; however, it could be interesting to consider \emph{learning} the adapter if this information were not known.

Our model focused on the problem of inter-annotator variability under inconsistent schemas.
However, due to the `plugin' nature of our adapter, the model is amenable to extensions which take into account for example task difficulty \cite{AMD_013} or shared latent-structure across the annotators \cite{AMD_010}.

\subsection*{Acknowledgements}
We thank our collaborators at \acrlong{mrc}, especially Dr Sara Wells, Dr Pat Nolan, Dr Rasneer Sonia Bains and Dr Henrik Westerberg for providing and explaining the data set.
MC's work was supported in part by the EPSRC Centre for Doctoral Training in Data Science, funded by the UK Engineering and Physical Sciences Research Council (grant EP/L016427/1) and the University of Edinburgh. 
The work of CW was supported in part by The Alan Turing Institute under the EPSRC grant EP/N510129/1.

\subsection*{Ethical Approval}
While none of the authors were involved in the data collection, all procedures and animal studies in the behavioural-characterisation data-set used were carried out in accordance with the Animals (Scientific Procedures) Act 1986, UK, Amendment Regulations 2012 (SI 4 2012/3039) as indicated in \cite{CBD_026}.

\appendix

\section{Extended Proofs}
\begin{lemma}
\label{LEM_SUPER_SPACE}
Let $Y_k^{'(n)}$ be a 1-Hot encoded variable, where the sample-space is denoted $L_{S^{(n)}}$: \ie this may vary between samples/annotators. This is equivalent to representing $Y_k^{(n)}$ by a fixed sample-space, where the probability of emitting $Y_y^{(n)}$ for some $y \notin L_{S^{(n)}} = 0$.
\end{lemma}

\begin{proof}
\begin{align*}
M_{\omega}^{(n)}\left(k,u\right) & = \prod_{s=1}^{|S|} \left[ \prod_{y=1}^{|Y|} \omega_{y,u,s}^{Y_{k,y}^{(n)}} \right]^{S_{k,s}^{(n)}} = \prod_{s=1}^{|S|} \left[ \prod_{y\in Y^{(n)}} \omega_{y,u,s}^{Y_{k,y}^{(n)}} \prod_{y\notin Y^{(n)}} \omega_{y,u,s}^{Y_{k,y}^{(n)}} \right]^{S_s^{(n)}}
\end{align*}
However, for the second set of products, $Y_{k,y}^{(n)} = 0$ by definition, since it is never observed. Hence,
\begin{align*}
M_{\omega}^{(n)}\left(k,u\right) & = \prod_{s=1}^{|S|} \left[ \prod_{y\in Y^{(n)}} \omega_{y,u,s}^{Y_{k,y}^{(n)}} \times \bm{1} \right]^{S_{k,s}^{(n)}} = \prod_{s=1}^{|S|} \left[ \prod_{y\in Y^{(n)}} \omega_{y,u,s}^{Y_{k,y}^{(n)}} \right]^{S_{k,s}^{(n)}}
\end{align*}
\end{proof}
Note that while we do not require that $\omega_{y,u,s}=0\  \forall y \notin L_{S^{(n)}}$, this is enforced to avoid the model expending probability mass on impossible combinations.

\section{Additional comparisons between \acrshort{isac} and \acrshort{nam}}
\label{APP_ABLATION_STUDY}

We evaluated the Extended and baseline \gls{nam} models through the macro F1-score, raw accuracy and predictive log-likelihood in synthetic experiments.
While the F1 and likelihood scores provide the best comparison, the accuracy is also reported as a more challenging baseline to beat (since it is generally easy to achieve high accuracy using poor models on very unbalanced datasets such as ours).
In each experiment, we simulated 20 independent runs, and evaluated the metrics on a hold-out set using 10-fold cross validation.
For most of the experiments, we used a \textsl{reduced} data-set size of 60 segments of 100 samples each, allowing us to test various conditions quickly (we show that this alone does not significantly impact our results, see first and second rows of Table \ref{TAB_ABLATION}).
We tested the effect of a \textsl{uniform} distribution over latent-states, $\Pi$ sampled (once) from a \textsl{Dirichlet} prior with $\alpha=10$ as well as schema distributions \textsl{Biased} towards the less informative ones (in the ratio 1:10:1:10).
Note that in this latter case, the number of segments was increased to 80 (since otherwise certain schemas do not appear in some folds).
$p$-Values corresponding to paired t-tests with 19-\gls{dof} (20 independent runs) are reported in all cases.

\begin{table}[!ht]
\centering
\caption{Predictive log-likelihood, F1 and accuracy for the \gls{isac} and \gls{nam} models under different conditions. In the interest of avoiding clutter, we omit $p$-Values for the accuracy, but in all cases, it was less than 1$\times10^{-16}$}
\begingroup
\setlength{\tabcolsep}{3pt}
\input{table_isar_ablation.tex}
\endgroup
\label{TAB_ABLATION}
\end{table}

\section{Parameter Recovery Curves}
\label{APP_PARAM_RECOVERY}

We carried out simulation experiments of the ability of the model to recover the `true' parameters, under a number of scenarios.
In each case, datasets were generated according to the parameters as learnt from the \gls{mrc} data, and subsequently we retrained the model from scratch, using successively larger portions of the dataset.
Each experiment was repeated 20 times, with noisy perturbation ($\sim \text{Unif}[0, 0.05]$) in the underlying prior/emission probabilities.
We evaluated the quality of the estimate with the \gls{rad} between the true $\Theta$ and learnt $\hat{\Theta}$ parameters, using the mean magnitude of individual probabilities as the normaliser:
$$ \gls{rad} = \frac{|\hat{\Theta} - \Theta| \times 100 \%}{mean(\Theta)} .$$
We prefer this over the KL-Divergence as it is more readily interpretable.

In the first case (Fig.\ \ref{FIG_EXTREME_RAD}) we experimented with an extreme scenario where each schema indicates only the presence/absence of a single label (\ie a One-vs-Rest schema).
To reduce the complexity of the problem, we generated data using the first six of the original 11 annotators, seven of the original 13 labels, and with maximum sample sizes of 500 segments of 100 time-points each.
Annotators and schemas were drawn from uniform distributions.
We investigated scenarios where (a) all responsible annotators use the same schema within a sample, and (b) annotators may use different schemas even within the same sample.

\begin{figure}[!ht]
\centering
	\begin{subfigure}[b]{0.49\columnwidth}
       	\includegraphics[width=\textwidth]{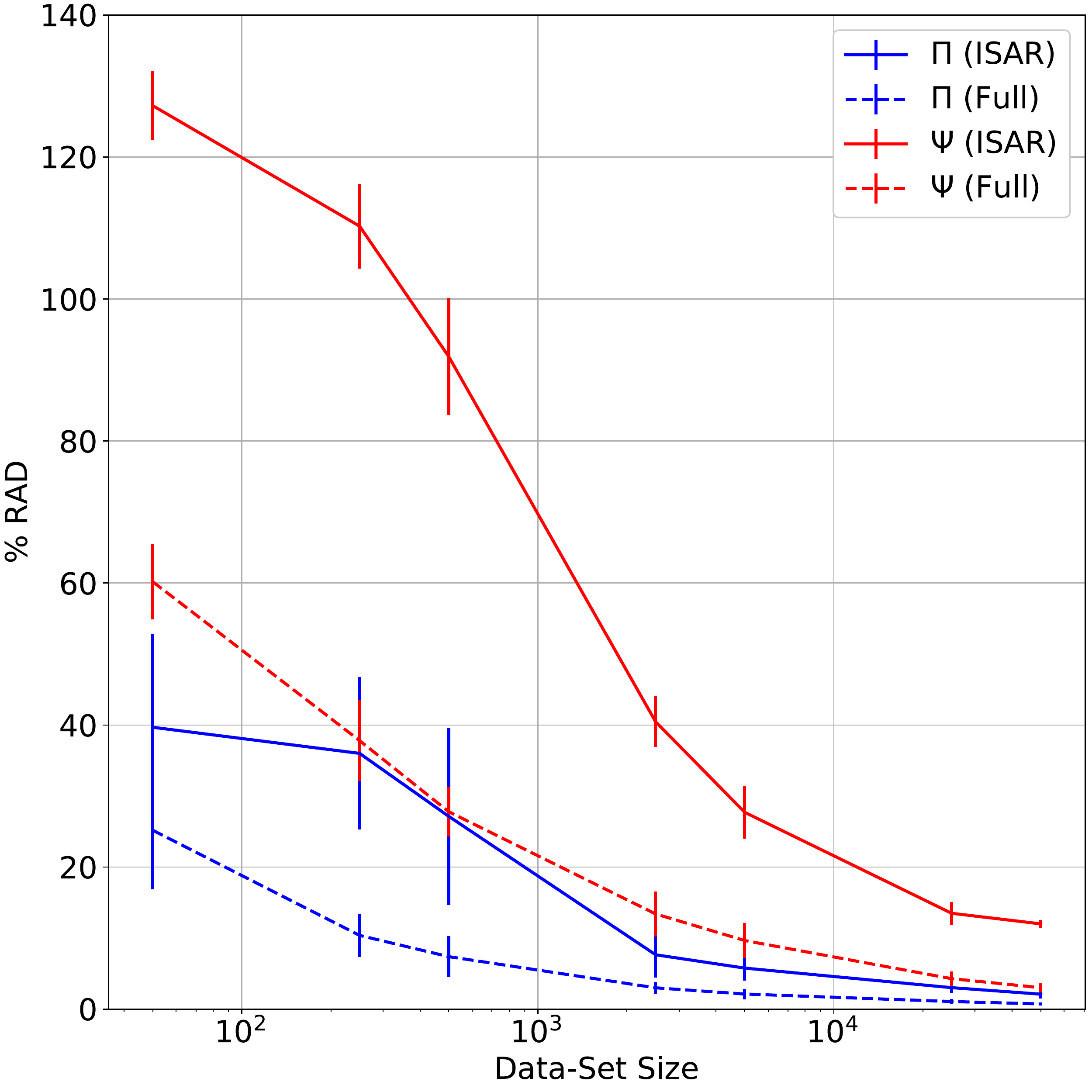}
        \caption{Same Schema per Annotator}
    \end{subfigure}
	\begin{subfigure}[b]{0.49\columnwidth}
       	\includegraphics[width=\textwidth]{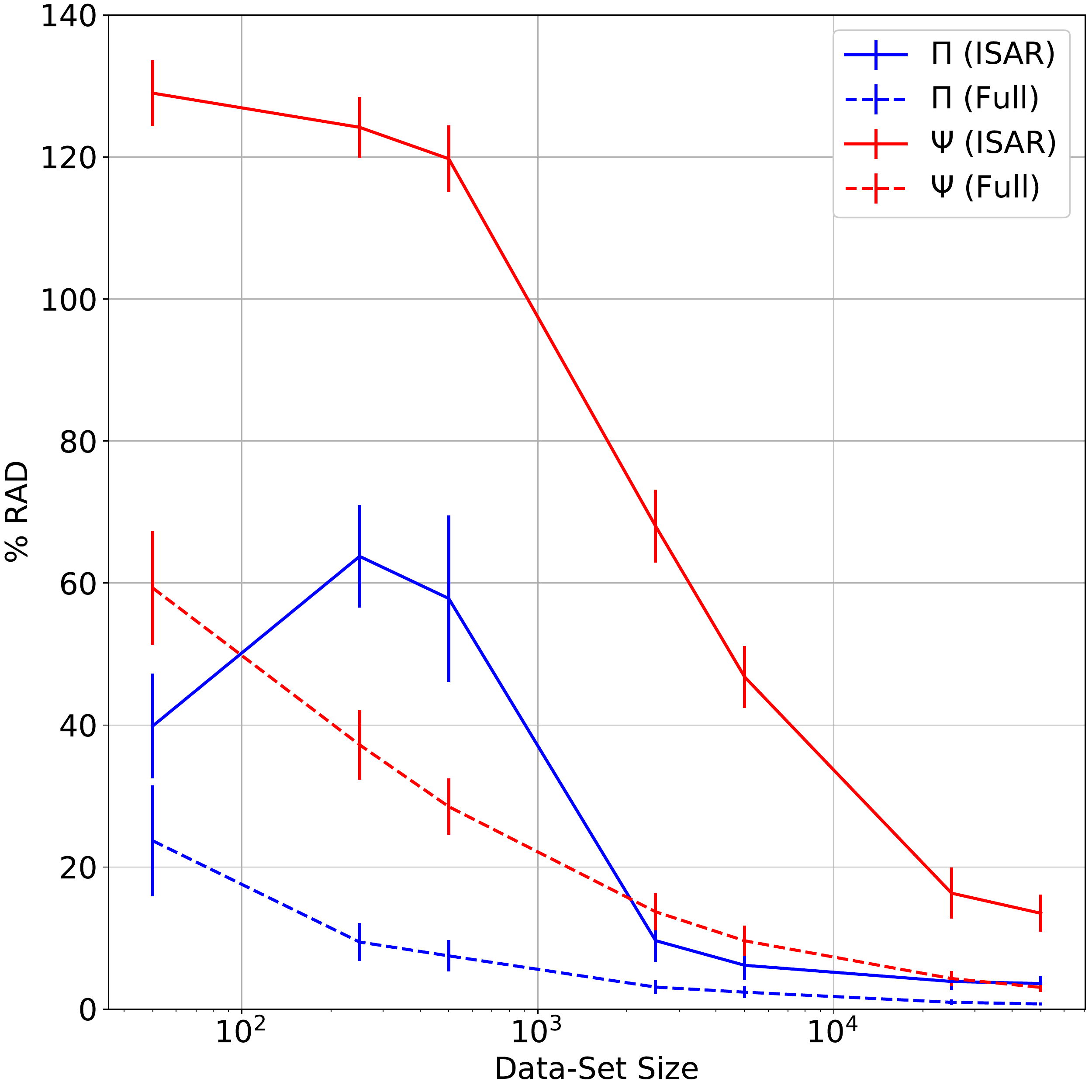}
        \caption{Potentially Different Schema}
    \end{subfigure}
\caption{\acrshort{rad} as a function of data-set size for the One-vs-Rest schemas. The error-bars indicate one standard deviation across runs. The initial increase in error in (b) is due to the interplay between the `prior' counts becoming insignificant, but there not being enough data to get a true estimate of the probabilities (due to label imbalance).}
\label{FIG_EXTREME_RAD}
\end{figure}

In the second case we used the same setup as in the real data, \ie the four schemas in the \gls{mrc} dataset, the full annotator/label-set and the full data-set size. This is shown in Fig.\ \ref{FIG_RAD_ACTUAL_SYNTHETHIC}.

\begin{figure}
\centering
	\begin{subfigure}[b]{0.49\textwidth}
       	\includegraphics[width=\textwidth]{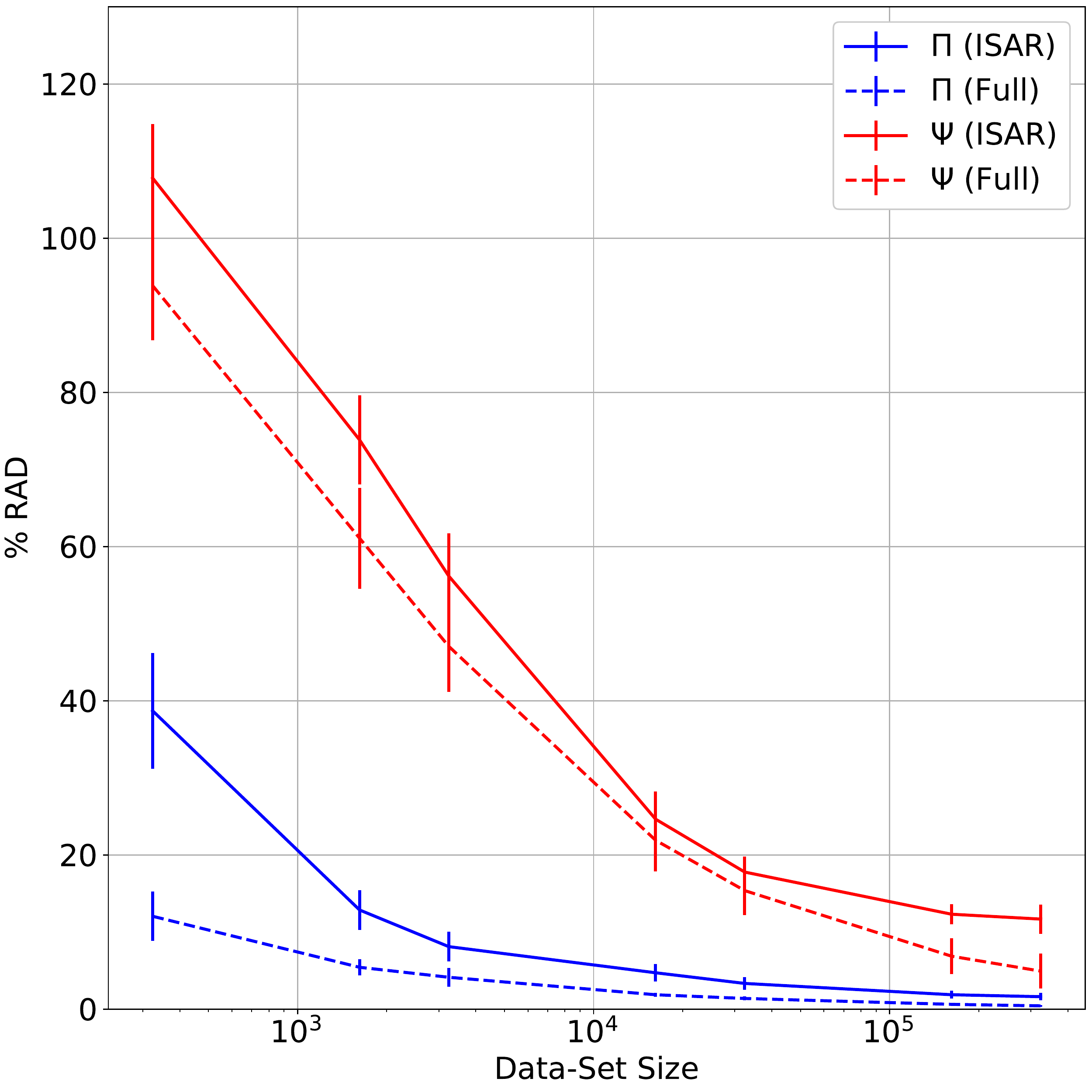}
        \caption{\gls{rad} for $\pi$ and $\psi$}
    \end{subfigure}
	\begin{subfigure}[b]{0.49\textwidth}
       	\includegraphics[width=\textwidth]{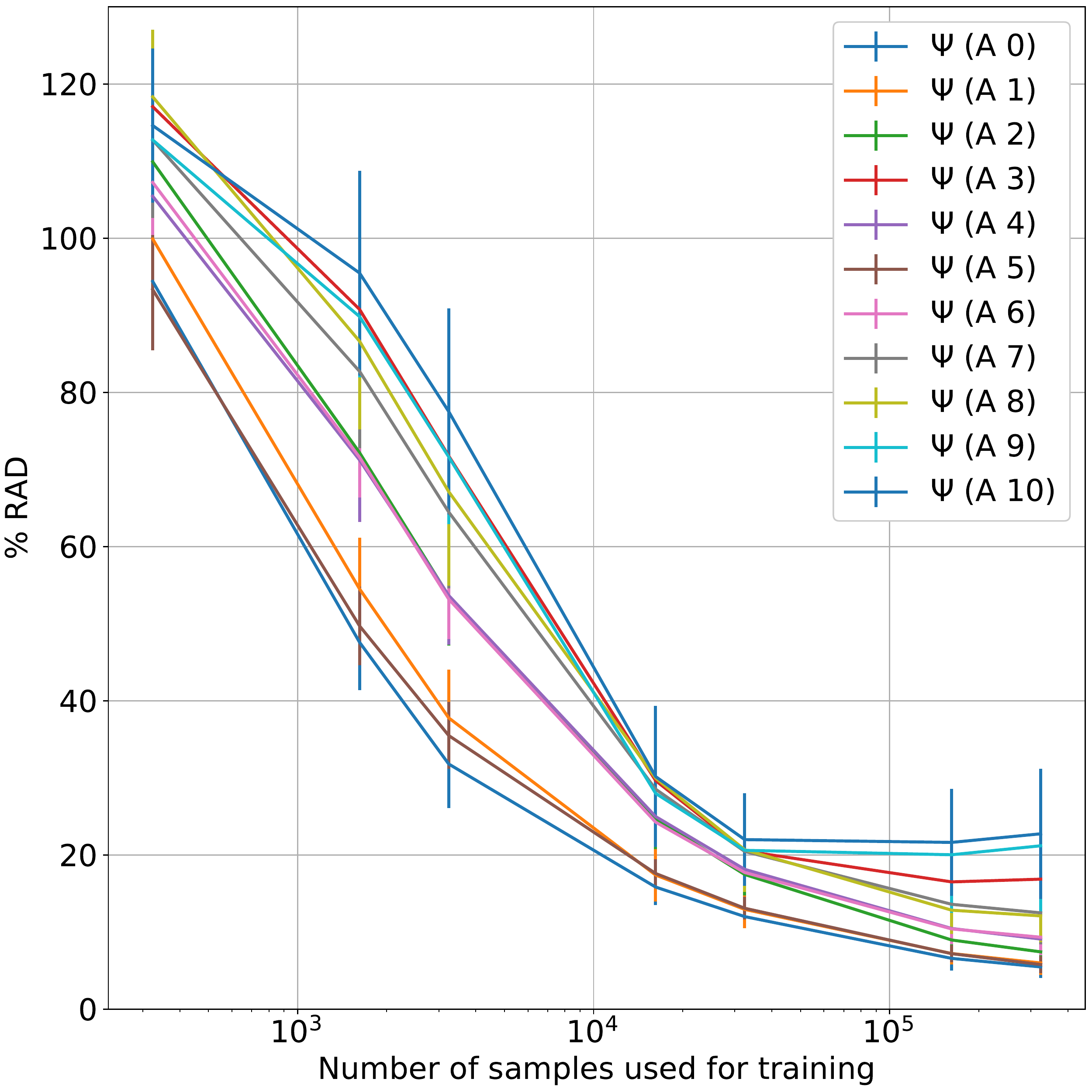}
        \caption{\gls{rad} per $\psi_k$}
    \end{subfigure}
\caption{\gls{rad} for Simulation runs based on actual data parameters.}
\label{FIG_RAD_ACTUAL_SYNTHETHIC}
\end{figure}

\end{document}

%% file: tikzbayes.code.tex
\usetikzlibrary{shapes}
\usetikzlibrary{fit}
\usetikzlibrary{chains}
\usetikzlibrary{arrows}

\tikzstyle{latent} = [circle,fill=white,draw=black,inner sep=2pt,
minimum size=30pt, font=\fontsize{10}{10}\selectfont, node distance=1]
\tikzstyle{latent_big} = [latent, minimum size=40pt]
\tikzstyle{obs} = [latent,fill=gray!25]
\tikzstyle{obs_big} = [obs, minimum size=45pt]
\tikzstyle{const} = [rectangle, draw=black, inner sep=2pt, node distance=1, fill=gray!25, minimum size=15pt]
\tikzstyle{blank} = [rectangle, inner sep=0pt, node distance=1]
\tikzstyle{factor} = [rectangle, fill=black, minimum size=5pt, inner
sep=0pt, node distance=0.4]
\tikzstyle{det} = [latent, diamond]

\tikzstyle{plate} = [draw, rectangle, rounded corners, fit=#1, inner sep=5pt]
\tikzstyle{wrap} = [inner sep=0pt, fit=#1]
\tikzstyle{gate} = [draw, rectangle, dashed, fit=#1]

\tikzstyle{caption} = [font=\footnotesize, node distance=0] %
\tikzstyle{plate caption} = [caption, node distance=0, inner sep=0pt,
right=5pt of #1.south east] %
\tikzstyle{factor caption} = [caption] %
\tikzstyle{every label} += [caption] %

\newcommand{\edge}[3][]{ %
  \foreach \x in {#2} { %
    \foreach \y in {#3} { %
      \path (\x) edge [->, >={triangle 45}, #1] (\y) ;%
    } ;
  } ;
}

\newcommand{\plate}[4][]{ %
  \node[wrap=#3] (#2-wrap) {}; %
  \node[plate caption=#2-wrap] (#2-caption) {#4}; %
  \node[plate=(#2-wrap)(#2-caption), #1] (#2) {}; %
}



%% file: glossaries.tex
\newacronym{nam}{DS}{Dawid-Skene}

\newacronym{isac}{ISAR}{Inter-Schema AdapteR}

\newacronym{cpt}{CPT}{Conditional Probability Table}

\newacronym{map}{MAP}{Maximum-A-Posteriori}

\newacronym{mle}{MLE}{Maximum-Likelihood}

\newacronym{em}{EM}{Expectation Maximisation}

\newacronym{tl}{TL}{Transfer Learning}

\newacronym{mtl}{MTL}{Multi-Task Learning}

\newacronym{hmlc}{HMlC}{Hierarchical Multi-label Classification}

\newacronym{mrc}{MRC Harwell}{the Medical Research Council, Harwell Institute, Oxfordshire}

\newacronym{nis}{NIS}{Not In Schema}

\newacronym{mar}{MAR}{Missing at Random}

\newacronym{dof}{DoF}{Degrees of Freedom}

\newacronym{mi}{MI}{Mutual Information}

\newacronym{iid}{IID}{independent and identically distributed}

\newacronym{rad}{RAD}{Relative Absolute Deviation}

%% file: model_ds.tex

\begin{tikzpicture}

  \node[latent]                				(zn) 	{$Z^{(n)}$};
  \node[obs, below=of zn] 		    		(ukn) 	{$U^{(n)}_{k}$};
  \node[const, above right=of zn]			(pi)  	{$\pi$};
  \node[const, above right=2.5cm of ukn] 	(psi)   {$\psi$};

  \edge {zn} 	{ukn};
  \edge {pi} 	{zn};
  \edge {psi} 	{ukn};

  \plate {pK} {(ukn)}		{$K$};
  \plate {pN} {(pK)(zn)} 	{$N$};

\end{tikzpicture}

%% file: model_isac.tex

\begin{tikzpicture}

  \node[latent]                		(zn) 	{$Z^{(n)}$};
  \node[latent, below=0.5cm of zn] 		(ukn) 	{$U^{(n)}_{k}$};
  \node[obs, below=0.5cm of ukn]			(ykn)	{$Y^{(n)}_{k}$};
  \node[obs, left=1.0cm of ykn]			(sn)	{$S^{(n)}_{k}$};
  \node[const, above right=of zn]	(pi)  	{$\pi$};
  \node[const, above right=2.5cm of ukn] 	(psi)   {$\psi$};
  \node[const, above right=2.5cm of ykn]	(omega)	{$\omega$};

  \edge {zn} 	{ukn};
  \edge {ukn}	{ykn};
  \edge {sn}	{ykn};
  \edge {pi} 	{zn};
  \edge {psi} 	{ukn};
  \edge {omega}	{ykn};

  \plate {pK} {(ykn)(ukn)(sn)}	{$K$};
  \plate {pN} {(pK)(zn)} 	{$N$};

\end{tikzpicture}

%% file: table_mutualinfo.tex
\begin{tabular}{@{}lcccccc@{}} \toprule[1.5pt]
			& \multicolumn{3}{c}{$\minfo{Z;\bm{Y}}$}	& \multicolumn{3}{c}{$R(Z;\bm{Y})$}	\\ \cmidrule(lr){2-4} \cmidrule(l){5-7}
			& Mean		& Min		& Max		& Mean		& Min		& Max		\\ \midrule
$Y = U$		& 1.48		& 1.39		& 1.56											\\ \midrule
 I			& 1.48		& 1.39		& 1.56									 		\\
 II			& 0.87		& 0.73		& 0.93											\\
 III		& 1.47		& 1.38		& 1.56											\\
 IV			& 0.95		& 0.90		& 0.99											\\ \midrule
 I+III		& 1.73		& 1.67		& 1.78		& 1.22		& 1.09		& 1.34		\\
 II+IV		& 1.10		& 1.06		& 1.15		& 0.72		& 0.60		& 0.80		\\ 
 I+III+IV	& 1.78		& 1.71		& 1.82		& 2.12		& 1.93		& 2.29		\\ \midrule
 ALL		& 1.79		& 1.72		& 1.82		& 2.98		& 2.73		& 3.21		\\ \bottomrule[1.5pt]
\end{tabular}

%% file: table_isar_ablation.tex

\begin{tabular}{@{}llcccccccc@{}} \toprule[1.5pt]
 						 	&				& \multicolumn{3}{c}{Log-Likelihood}						                & \multicolumn{3}{c}{Macro F1}	                                            & \multicolumn{2}{c}{Accuracy(\%)}  \\\cmidrule(l){3-5}\cmidrule(l){6-8}\cmidrule(l){9-10}
 						 	&				& \textbf{Mean}	    & \textbf{Std}	& \textbf{\textit{p}-value}	            & \textbf{Mean}	    & \textbf{Std}	& \textbf{\textit{p}-value}	            & \textbf{Mean} & \textbf{Std}  \\\midrule
\multirow{2}{*}{Realistic}  & \gls{nam} 	& -0.61			    & 0.09			& \multirow{2}{*}{6.1$\times10^{-16}$}  & 0.785			    & 0.02			& \multirow{2}{*}{3.2$\times10^{-26}$}  & 82.5          & 2.13          \\
			    		 	& \gls{isac} 	& \textbf{-0.35}	& \textbf{0.04} &  		                                & \textbf{0.864}	& 0.02	        & 		                                & \textbf{84.5} & \textbf{1.95} \\\midrule
\multirow{2}{*}{Reduced} 	& \gls{nam} 	& -0.60			    & 0.09			& \multirow{2}{*}{1.0$\times10^{-14}$}  & 0.757			    & 0.02			& \multirow{2}{*}{2.5$\times10^{-14}$}	& 82.6          & 2.05          \\
			    		 	& \gls{isac} 	& \textbf{-0.37} 	& \textbf{0.04} & 		                                & \textbf{0.803}	& 0.02	        & 		                                & \textbf{84.3} & \textbf{1.84} \\\midrule
\multirow{2}{*}{Uniform} 	& \gls{nam} 	& -0.92			    & 0.08			& \multirow{2}{*}{8.0$\times10^{-24}$}  & 0.726			    & 0.02			& \multirow{2}{*}{1.9$\times10^{-21}$}	& 73.8          & 1.99          \\
			    		 	& \gls{isac} 	& \textbf{-0.38} 	& \textbf{0.04} & 		                                & \textbf{0.840}    & 0.02	        & 		                                & \textbf{83.4} & \textbf{1.89} \\\midrule
\multirow{2}{*}{Dirichlet} 	& \gls{nam} 	& -0.94			    & 0.07  		& \multirow{2}{*}{1.6$\times10^{-24}$}  & 0.727			    & 0.02			& \multirow{2}{*}{2.2$\times10^{-23}$}	& 73.8          & 1.83          \\
			    		 	& \gls{isac} 	& \textbf{-0.36} 	& \textbf{0.04} & 		                                & \textbf{0.835}	& 0.02	        & 		                                & \textbf{84.8} & \textbf{1.75} \\\midrule
\multirow{2}{*}{Biased}   	& \gls{nam} 	& -1.83			    & 0.21			& \multirow{2}{*}{3.6$\times10^{-19}$}  & 0.533			    & 0.02			& \multirow{2}{*}{1.6$\times10^{-24}$}	& 35.7          & 2.32          \\
			    		 	& \gls{isac} 	& \textbf{-0.82} 	& \textbf{0.15} & 		                                & \textbf{0.667}	& \textbf{0.01}	& 		                                & \textbf{69.3} & \textbf{1.21} \\\midrule
Biased \& 				   	& \gls{nam} 	& -1.60			    & \textbf{0.11}	& \multirow{2}{*}{9.1$\times10^{-10}$}  & 0.585			    & 0.02			& \multirow{2}{*}{5.5$\times10^{-19}$}	& 62.5          & \textbf{1.05} \\
Uniform		    		 	& \gls{isac} 	& \textbf{-0.82} 	& 0.37          & 		                                & \textbf{0.709}	& 0.02	        & 		                                & \textbf{71.9} & 1.34          \\\bottomrule[1.5pt]
\end{tabular}

%% file: Paper.bbl
\begin{thebibliography}{10}

\bibitem{MISC_023}
Christopher~M. Bishop.
\newblock {\em {Pattern Recognition and Machine Learning}}.
\newblock Springer, New York, 2006.

\bibitem{MISC_015}
Olga Russakovsky, Jia Deng, Hao Su, Jonathan Krause, Sanjeev Satheesh, Sean Ma,
  Zhiheng Huang, Andrej Karpathy, Aditya Khosla, Michael Bernstein, Alexander~C
  Berg, and Li~Fei-Fei.
\newblock {ImageNet Large Scale Visual Recognition Challenge}.
\newblock {\em International Journal of Computer Vision (IJCV)},
  115(3):211--252, 2015.

\bibitem{AMD_012}
A.~P. Dawid and A.~M. Skene.
\newblock {Maximum Likelihood Estimation of Observer Error-Rates Using the EM
  Algorithm}.
\newblock {\em Applied Statistics}, 28(1):20--28, 1979.

\bibitem{AMD_019}
Vikas~C. Raykar, Shipeng Yu, Linda~H. Zhao, Gerardo~Hermosillo Valadez, Charles
  Florin, Luca Bogoni, and Linda Moy.
\newblock {Learning From Crowds}.
\newblock {\em The Journal of Machine Learning Research}, 11(Apr):1297----1322,
  2010.

\bibitem{AMD_017}
Ding Zhou, Jiang Bian, Shuyi Zheng, Hongyuan Zha, and C.~Lee Giles.
\newblock {Exploring Social Annotations for Information Retrieval}.
\newblock In {\em Proceedings of the 17th International Conference on World
  Wide Web}, pages 715----724, Beijing, China, 2008. ACM.

\bibitem{AMD_015}
Geoffrey~J McLachlan and Thriyambakam Krishnan.
\newblock {\em {The EM Algorithm and Extensions}}, volume~54.
\newblock Wiley-Interscience, New York, 2 edition, 2008.

\bibitem{MISC_009}
Michael P.~J. Camilleri.
\newblock {\em {Modelling Annotator Variability across Feature Spaces in the
  Temporal Analysis of Behaviour}}.
\newblock M.sc. dissertation, University of Edinburgh, 2018.

\bibitem{TML_013}
Sinno~Jialin Pan and Qiang Yang.
\newblock {A Survey on Transfer Learning}.
\newblock {\em IEEE Transactions on Knowledge and Data Engineering},
  22(10):1345--1359, oct 2010.

\bibitem{TML_012}
Yu~Zhang and Qiang Yang.
\newblock {A Survey on Multi-Task Learning}.
\newblock {\em Arxiv}, pages 1--20, jul 2017.

\bibitem{AMD_023}
Daniel~F. Heitjan and Donald~B. Rubin.
\newblock {Ignorability and Coarse Data}.
\newblock {\em The Annals of Statistics}, 19(4):2244 ---- 2253, dec 1991.

\bibitem{AMD_029}
Timothee Cour, Ben Sapp, and Ben Taskar.
\newblock {Learning from Partial Labels}.
\newblock {\em Journal of Machine Learning Research}, 12:1501--1536, jul 2011.

\bibitem{MISC_010}
Carlos~N. {Silla Jr.} and Alex~A. Freitas.
\newblock {A survey of hierarchical classification across different application
  domains}.
\newblock {\em Data Mining and Knowledge Discovery}, 22(1-2):31 ---- 72, jan
  2011.

\bibitem{MISC_025}
Jonatas Wehrmann, Ricardo Cerri, and Rodrigo Barros.
\newblock {Hierarchical Multi-Label Classification Networks}.
\newblock In Jennifer Dy and Andreas Krause, editors, {\em Proceedings of the
  35th International Conference on Machine Learning}, volume~80 of {\em
  Proceedings of Machine Learning Research}, pages 5075----5084. PMLR, 2018.

\bibitem{CBD_026}
Rasneer~S. Bains, Heather~L. Cater, Rowland~R. Sillito, Agisilaos Chartsias,
  Duncan Sneddon, Danilo Concas, Piia Keskivali-Bond, Timothy~C. Lukins, Sara
  Wells, Abraham {Acevedo Arozena}, Patrick~M. Nolan, and J.~Douglas Armstrong.
\newblock {Analysis of Individual Mouse Activity in Group Housed Animals of
  Different Inbred Strains using a Novel Automated Home Cage Analysis System}.
\newblock {\em Frontiers in Behavioral Neuroscience}, 10:106, jun 2016.

\bibitem{MISC_011}
Roderick J.~A. Little and Donald~B. Rubin.
\newblock {\em {Statistical Analysis with Missing data}}.
\newblock 2 edition, 2014.

\bibitem{AMD_021}
M.~Sperrin, T.~Jaki, and E.~Wit.
\newblock {Probabilistic relabelling strategies for the label switching problem
  in Bayesian mixture models}.
\newblock {\em Statistics and Computing}, 20(3):357--366, jul 2010.

\bibitem{MISC_013}
Kevin~Patrick Murphy.
\newblock {\em {Machine learning: a probabilistic perspective}}.
\newblock MIT Press, 2012.

\bibitem{MISC_012}
Nicholas Timme, Wesley Alford, Benjamin Flecker, and John~M. Beggs.
\newblock {Synergy, redundancy, and multivariate information measures: an
  experimentalist's perspective}.
\newblock {\em Journal of Computational Neuroscience}, 36(2):119--140, apr
  2014.

\bibitem{AMD_013}
Jacob Whitehill, Paul Ruvolo, Tingfan Wu, Jacob Bergsma, and Javier Movellan.
\newblock {Whose Vote Should Count More: Optimal Integration of Labels from
  Labelers of Unknown Expertise}.
\newblock In Yoshua Bengio, D.~Schurrmans, J.~D. Lafferty, C.~K.~I. Williams,
  and A.~Culotta, editors, {\em Advances in Neural Information Processing
  Systems 22}, volume~22, pages 2035--2043. Curran Associates, Inc., 2009.

\bibitem{AMD_010}
Fabian~L. Wauthier and Michael~I. Jordan.
\newblock {Bayesian Bias Mitigation for Crowdsourcing}.
\newblock In J.~Shawe-Taylor, R.~S. Zemel, P.~L. Bartlett, F.~Pereira, and
  K.~Q. Weinberger, editors, {\em Advances in Neural Information Processing
  Systems 24}, pages 1800--1808, Granada, Spain, 2011.

\end{thebibliography}
